\documentclass[10pt,twocolumn,twoside,conference]{ieeeconf}
\IEEEoverridecommandlockouts
\usepackage{epsfig} 
\usepackage{times} 
\usepackage{amsmath} 
\usepackage{amssymb}  
\usepackage{enumerate}
\usepackage{cite}
\usepackage{slashbox}
\usepackage{wrapfig}
\usepackage{color}
\usepackage{romannum}
\usepackage{xspace}
\usepackage{algorithmic}
\usepackage[ruled]{algorithm}
\usepackage{subfigure}
\usepackage{graphicx}
\usepackage{epstopdf}
\usepackage{float}
\usepackage{bm}
\usepackage{comment}
\DeclareGraphicsExtensions{.png}

\newtheorem{lemma}{Lemma}
\newtheorem{theorem}{Theorem}

\newtheorem{rem}{Remark}
\newtheorem{assumption}{Assumption}

\DeclareMathOperator{\prox}{\mathbf{prox}}

\title{\LARGE \bf Decentralized Statistical Inference with Unrolled\\ Graph Neural Networks }
\author{ He Wang, Yifei Shen, Ziyuan Wang, Dongsheng Li, Jun Zhang, Khaled B. Letaief and Jie Lu
\thanks{This work has been supported by the National Natural Science Foundation of China under grant 61603254 and Hong Kong Research Grant Council under Grant No. 16210719.}
\thanks{H. Wang, Z. Wang and J. Lu are with the School of Information Science and Technology, ShanghaiTech University, 201210 Shanghai, China. Email: {\tt \{wanghe, wangzy11,lujie\}@shanghaitech.edu.cn}. Y. Shen and K. B. Letaief are with the Department of Electronic and Computer Engineering, The Hong Kong University of Science and Technology, Hong Kong, China. Email: {\tt \{yshenaw, eekhaled\}@ust.hk}. D. Li is with Microsoft Research Asia, Shanghai, China. Email: {\tt dongshengli@fudan.edu.cn}. J. Zhang is with the Department of Electronic and Information Engineering, The Hong Kong Polytechnic University, Hong Kong, China. Email: {\tt jun-eie.zhang@polyu.edu.hk}.}
}
\begin{document}

\maketitle
\thispagestyle{empty}
\pagestyle{empty}
	
	\begin{abstract}
In this paper, we investigate the decentralized statistical inference problem, where a network of agents cooperatively recover a (structured) vector from private noisy samples without centralized coordination. Existing optimization-based algorithms suffer from issues of model mismatch and poor convergence speed, and thus their performance would be degraded, provided that the number of communication rounds is limited. This motivates us to propose a learning-based framework, which unrolls well-noted decentralized optimization algorithms (e.g., Prox-DGD and PG-EXTRA) into graph neural networks (GNNs). By minimizing the recovery error via end-to-end training, this learning-based framework resolves the model mismatch issue. Our convergence analysis (with PG-EXTRA as the base algorithm) reveals that the learned model parameters may accelerate the convergence and reduce the recovery error to a large extent. The simulation results demonstrate that the proposed GNN-based learning methods prominently outperform several state-of-the-art optimization-based algorithms in convergence speed and recovery error.
	\end{abstract}
	
	%
	%
	\section{Introduction}\label{sec:introduction}

We consider the decentralized statistical inference problem that a network of agents collaboratively recover a (structured) vector $x^*$ from individual noisy measurements $y_i = f_i(x^*) + \epsilon_i$, where $f_i$ is a sampling function and $\epsilon_i$ is the measurement noise at agent $i$. Such a problem has a considerable number of applications in decentralized machine learning \cite{he2018cola,talwar2015nearly} and decentralized communication systems \cite{zaib2016distributed,guo2020sparse}, and it has to be solved in a real-time manner given strict latency requirements and limited communication budgets.

This statistical inference problem is often formulated as composite programming, and a substantial number of decentralized optimization algorithms have been proposed to solve it \cite{shi2015pgextra,zeng2017fast,zeng2018proxdgd,li2019nids,xu2020abc,wu2020iplux}. While a linear or sublinear convergence rate has been rigorously established in these works, directly applying these algorithms suffers from two main drawbacks. First, there is a model mismatch between recovering the ground-truth vector and optimizing the composite programming, as the latter is often based on convex relaxation. Although some studies demonstrate the near-optimal performance of these formulations \cite{chandrasekaran2012convex}, the analysis is often built upon careful choices of parameters and some strict random pattern of $f_i$, which are difficult to satisfy in practice and thus lead to the unsatisfactory performance \cite{candes2008enhancing}. The second issue is the slow convergence speed of the existing algorithms. To converge, these algorithms often require hundreds of iterations, even with fine-tuned step-sizes. 

To address the above issues, a recent line of research applies data-driven approaches to learn the optimal recovery $x^*$ in an end-to-end manner \cite{gama2020graph,shen2020graph,tolstaya2021learning} via graph neural networks (GNNs). GNNs are neural networks built on distributed message passing \cite{xu2018powerful,sato2019approximation,loukas2019graph}. In each layer of a GNN, every node first aggregates features from its neighbors, and combines them with its own feature, which then undergoes a nonlinear component. On the other hand, in each iteration of a decentralized algorithm, each agent utilizes a \textit{weighted sum} aggregation function to mix its neighbors' information with its own decision variable, followed by other optimization components guiding the descent direction. Therefore, a $K$-layer GNN can naturally mimic a decentralized algorithm with $K$ communication rounds. Given a fixed number of communication rounds, we expect GNNs to learn an optimal update towards the ground-truth vector $x^*$ by adopting the recovery error as the loss function \cite{gama2020graph,loukas2019graph}. 

Despite the superior performance of GNNs in some specific settings, the convergence generally cannot be  guaranteed and their performance deteriorates when the dimension of the decision variable on each agent is larger than $10$ \cite{shen2020graph}. The key to improving the performance and scalability of GNNs is to introduce problem-specific inductive biases into the neural network architecture \cite{shrivastava2019glad,xu2019can,xu2020neural}. For example, when fulfilling the task of finding the maximal degree of a graph \cite{xu2020neural}, the GNN with the \textit{max} aggregation function is guaranteed to perform better than that with the \textit{sum} aggregation function. 

In this paper, we propose an unrolled GNN-based framework for real-time near-optimal decentralized statistical inference. In contrast to previous GNN-based decentralized algorithms, we use the unrolled first-order decentralized algorithm (e.g., PG-EXTRA) as the inductive bias for GNNs. The regularization and step-size are parameterized in an iteration-wise fashion, allowing them to be learned by minimizing the recovery error. In this way, the proposed framework not only enjoys the advantages of data-driven methods, but also inherits the performance guarantees of classic first-order decentralized algorithms. We further develop our proposed framework with decentralized LASSO and provide the convergence analysis with respect to the recovery error for PG-EXTRA.
Simulations demonstrate the fast convergence speed and superior performance in terms of recovery error and robustness to noises. In addition, thanks to the strong inductive bias, the proposed learning-based framework enjoys favorable performance with only tens of training samples.

	The outline of the paper is as follows: Section \ref{sec:sysmodelandmotivations} describes the formulation of the decentralized statistical inference problem. Section \ref{sec:relation and unrolling}  illustrates the similarity between decentralized algorithms and GNN architectures, and describes the unrolling guidelines. Section \ref{sec:algdevelop} 
	proposes the learning-based methods for the sparse vector recovery problem and shows the convergence analysis with respect to the recovery error for PG-EXTRA. Section \ref{sec:experiments} exhibits the simulation results. Section \ref{sec:conclusion} concludes the paper. The codes to reproduce the simulation results are available on https://github.com/IrisWangHe/Learning-based-DOP-Framework.
	
	\subsection*{Notation}  For any vector $x\in\mathbb{R}^n$, we use $\|x\|_0$, $\|x\|_1$, $\|x\|_2$  to denote $\ell_0$-norm, $\ell_1$-norm and  the Euclidean norm, respectively. We let $I$, $\mathbf{O}$, $\mathbf{1}$ and $\mathbf{0}$ represent the identity matrix, the all-zero matrix, the all-one vector and the all-zero vector of proper dimensions. For any matrix $A\in\mathbb{R}^{m\times n}$, $\|A\|_2$ and $\|A\|_F$ are its spectral norm and Frobenius norm respectively. In addition, $\text{null}\{A\}$ is the null space of $A\in\mathbb{R}^{m\times n}$ and $\text{span}\{x\}$ is the linear span of vector $x\in\mathbb{R}^n$. For any symmetric matrices $W,\tilde{W}\in\mathbb{R}^{n\times n}$, $W \succeq \tilde{W}$ means that $W-\tilde{W}$ is positive semidefinite. If $W$ is positive semidefinite, then $\lambda_{\max}(W)$ and $\lambda_{\min}(W)$ represent its largest and smallest eigenvalues respectively, and   $\|\mathbf{x}\|_W=\sqrt{\mathbf{x}^TW\mathbf{x}}$ represents the weighted norm for any $\mathbf{x}\in \mathbb{R}^{n\times d}$.
    Given vectors $x_1,\ldots,x_N\in\mathbb{R}^d$ and functions $f_1,\ldots,f_N:\mathbb{R}^d\to\mathbb{R}$, we let $\mathbf{x}=(x_1,\ldots,x_N)^T\in\mathbb{R}^{N\times d}$ represent the matrix obtained by compacting $x_1,\ldots,x_N$ and  $\mathbf{f}(\mathbf{x}) = \sum_{i=1}^N f_i(x_i)$, and then use  $\tilde{\nabla} \mathbf{f}(\mathbf{x}) = (\tilde{\nabla} f_1(x_1),\ldots,\tilde{\nabla} f_N(x_N))^T$ to denote
	a subgradient of $\mathbf{f}$ at  $\mathbf{x}$. Further, if all the $N$ functions are differentiable, then $\nabla \mathbf{f}(\mathbf{x})$ denotes the gradient of $\mathbf{f}$ at $\mathbf{x}$. 
	The proximal mapping of $\mathbf{f}$ is
    $$\prox_{\lambda \mathbf{f}}(\mathbf{x})=\underset{\mathbf{y} \in \mathbb{R}^{N\times d}}{\arg \min }~ \lambda \mathbf{f}(\mathbf{y})+\frac{1}{2}\|\mathbf{x}-\mathbf{y}\|_F^{2},$$ where $\lambda \ge 0$ is a scalar. The function $\mathbf{f}$ is proximable, if the proximal mapping of $\mathbf{f}$ has a closed form or can be computed efficiently. We use $I_X$ to denote the indicator function with respect to the set $X$, i.e., $I_X(x) = 0$ if $x\in X$ and $I_X(x)=+\infty$ otherwise. We also let $\text{ReLU}(\mathbf{x}) = \max\{\mathbf{x},\mathbf{O}\}$, where $\mathbf{x}\in\mathbb{R}^{N\times d}$ and $\max\{\mathbf{a},\mathbf{b}\}$ is an element-wise maximum operator between $\mathbf{a}$ and $\mathbf{b}$ of proper dimensions. We use $\mathcal{N}(\mu,\Sigma)$ to represent the multivariate normal distribution or Gaussian distribution, where $\mu$ 
   is the mean vector and $\Sigma$ is the covariance matrix of propose dimensions.
    The symbol $g(x) \leq \mathcal{O}_{\mathbb{P} }(f(x) )$ denotes that with high probability, there exists a constant $c$, such that $g(x) \leq c f(x)$.

	%
	%

\section{Problem Formulation}\label{sec:sysmodelandmotivations}
Suppose a set $\mathcal{V} = \{1,2,\ldots,N\}$  of agents  and their interactions form a networked multi-agent system, which is modeled as an undirected and connected \textit{communication graph} $\mathcal{G}=(\mathcal{V},\mathcal{E})$, where $\mathcal{E} \subseteq \{\{i,j\}:i,j\in\mathcal{V},i\ne j\}$ represents the set of communication links. Agent $i\in\mathcal{V}$ can only communicate   with its (one-hop) neighbors, denoted by $\mathcal{N}_i=\{j\in\mathcal{V}:\{i,j\}\in\mathcal{E}\}$. 

The goal of all the agents is to  cooperatively recover a (structured) vector $x^* \in \mathbb{R}^d$ from their noisy measurements. Each agent $i\in\mathcal{V}$ is associated with its own estimate vector $x_i\in\mathbb{R}^d$ of $x^*$ and its local measurement $y_i \in\mathbb{R}^{m_i}$, in the form of
\begin{equation*}
    y_i = f_i(x^*) + \epsilon_i,
\end{equation*}
where $f_i:\mathbb{R}^d\to\mathbb{R}^{m_i}$ is the measurement or sampling function of agent $i$, and $\epsilon_i\in \mathbb{R}^{m_i}$ is the measurement noise by agent $i$. This setting is prevalent in communication and machine learning. For example, in a distributed communication system, $x^*$ denotes the transmit symbol, $f_i$, $y_i$, $\epsilon_i$ denote the channel propagation, the received symbol, and the addictive noise of the channel at the $i$-th receiver, respectively \cite{zaib2016distributed,guo2020sparse}. In distributed private learning, $x^*$ is the optimal model weight, $f_i$ is the data-generating process, $y_i$ is the label and $\epsilon_i$ is the sampling noise in data \cite{he2018cola,talwar2015nearly}. 

Due to the limited number of measurements and the presence of noise, it is impossible for each agent to recover $x^*$ on its own. In the decentralized setting, each agent only maintains its own measurements and the recovering objective function privately. Unlike the centralized statistical inference approaches, there is no data fusion center that collects data or measurements from agents.  To recover $x^*$ in a coordinated way, each agent $i\in\mathcal{V}$ is allowed to communicate with its neighbors, such as its latest local estimate.

A typical approach is to formulate the recovery problem as the following composite programming problem:
\begin{equation}\label{prob:composite}
\begin{aligned}
&\underset{x_i \in \mathbb{R}^d\ \forall i\in\mathcal{V}}{\text{minimize}} & & \sum_{i=1}^N \left(s_i\left(x_i \right) + \lambda r_i(x_i)\right)\\
&\text{subject to}& & x_1 = \cdots = x_N,
\end{aligned} 
\end{equation}
where each $s_i(x_i) = h_i(f_i(x_i); y_i)$ is a function measuring the distance between $f_i(x_i)$ and $y_i$, each $r_i(x_i)$ is a regularizer revealing the underlying structure of the optimal solution $x^*$ (i.e., the vector to be recovered) and $\lambda$ is a nonnegative scalar. Suppose that each $s_i:\mathbb{R}^{d}\to \mathbb{R}$ is a convex differentiable function whose gradient is Lipschitz continuous, and each $r_i: \mathbb{R}^d\to \mathbb{R}\cup \{+\infty\}$ is a convex and proximable but possibly nondifferentiable function. Note that  $r_i$ is allowed to include an indicator function $I_{X_i}$, with respect to any closed convex set $X_i$, so that problem \eqref{prob:composite} can be constrained by $x_i\in X_i$.

Indeed, many real-world engineering problems can be cast into problem \eqref{prob:composite}. In the setting of decentralized compressive sensing, $s_i$ is squared $\ell_2$ distance, and $r_i$ is the $\ell_1$-norm (if $x^*$ is sparse) \cite{tibshirani1996regression} or nuclear norm (if $x^*$ is low-rank) \cite{candes2010power}. In the setting of distributed learning, $s_i$ is often the cross-entropy or support vector machine (SVM) loss, and $r_i$ is the $\ell_1$-norm or Euclidean norm.

\begin{figure*} 
\subfigure[An illustration of Prox-DGD in the compact form.]{
\begin{minipage}[t]{0.48\linewidth}
\includegraphics[width = 0.95\textwidth,trim=100 100 100 100,clip]{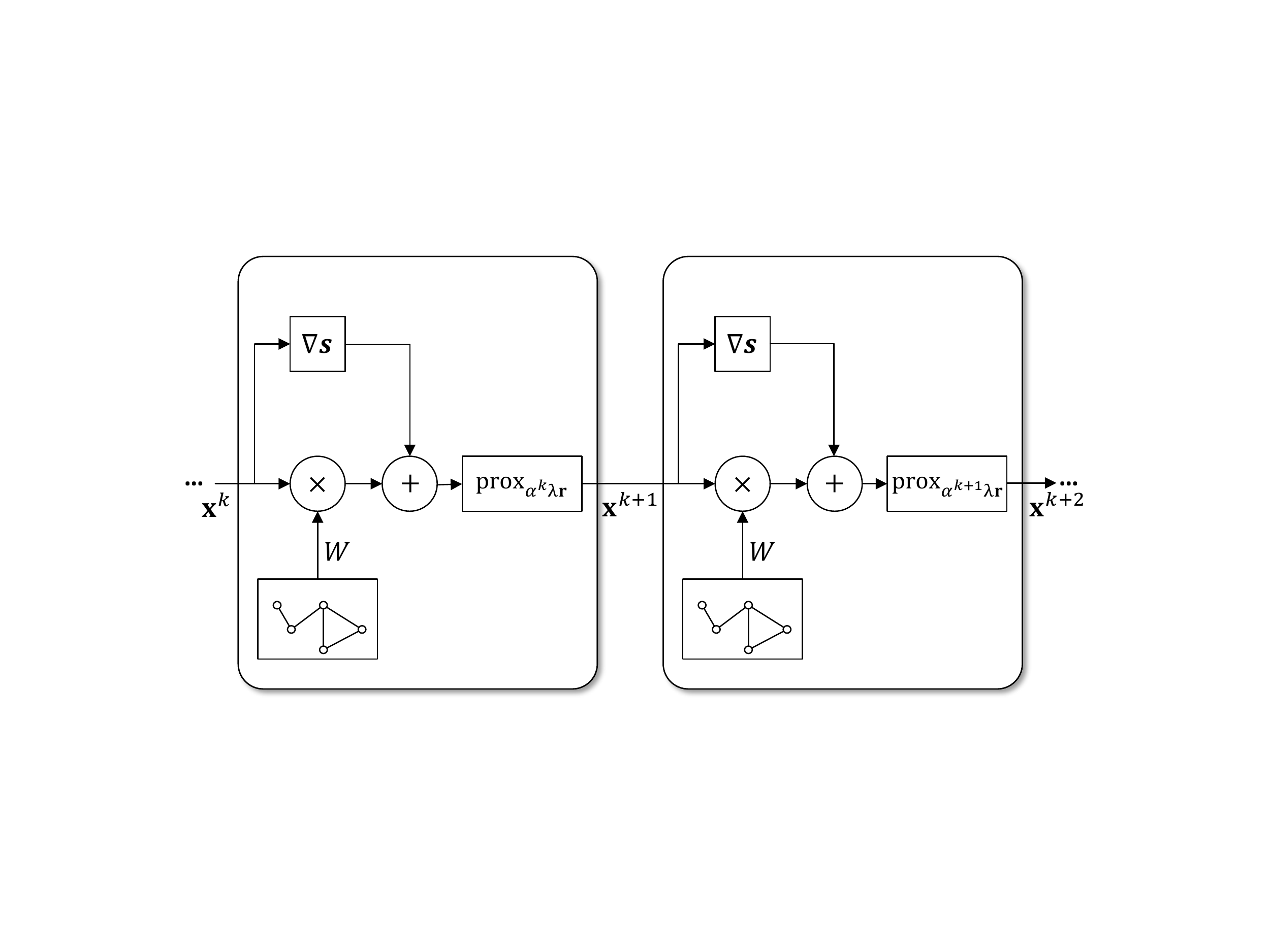}
\label{fig:dop}
\end{minipage}
}
\subfigure[An illustration of GCN in the compact form.]{
\begin{minipage}[t]{0.48\linewidth}
\centering
\includegraphics[width = 0.95\textwidth,trim=100 100 100 100,clip]{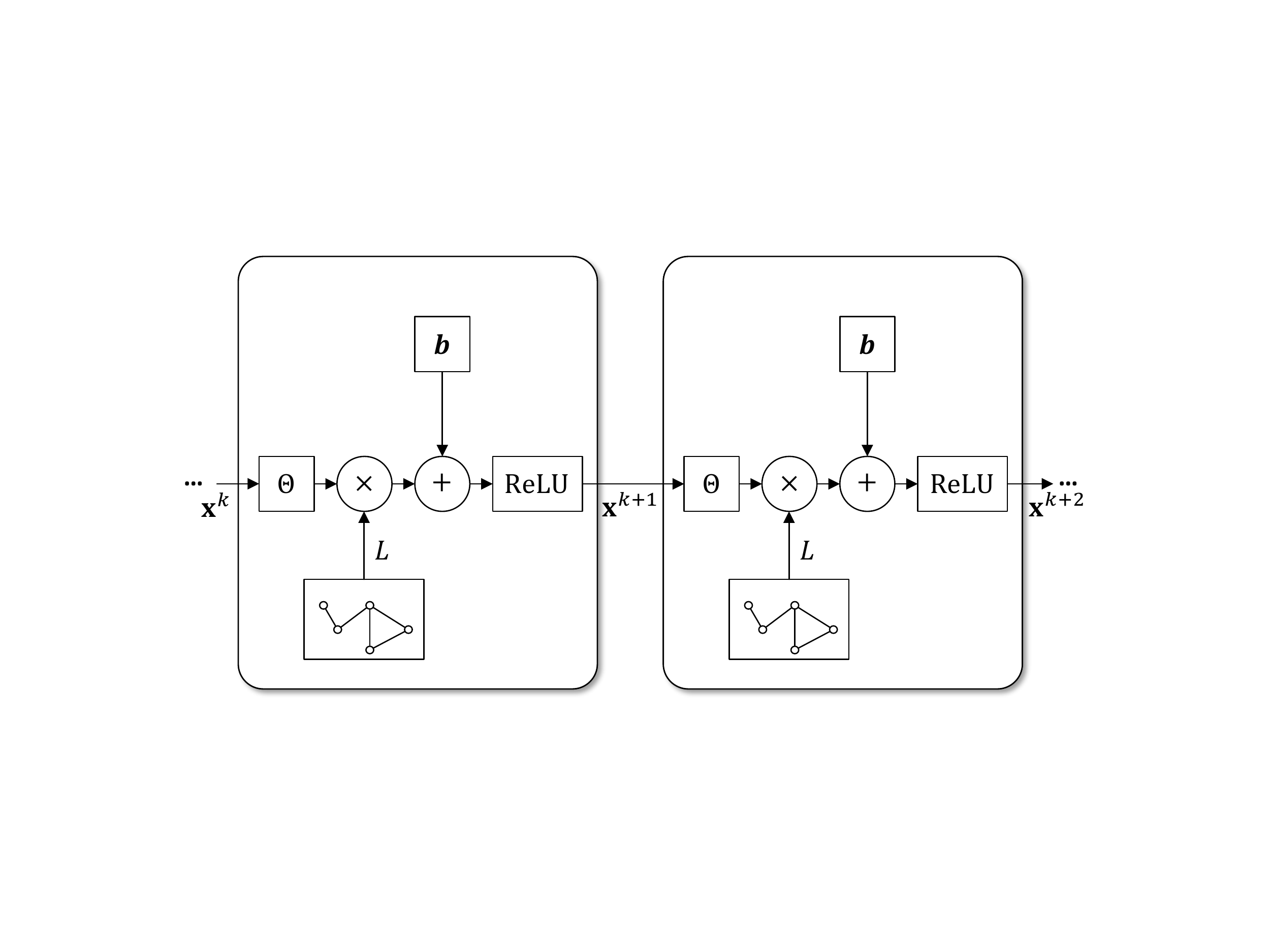}
\label{fig:gcn}
\end{minipage}
}
\caption{The similarities between Prox-DGD \cite{zeng2018proxdgd} and the GCN  \cite{kipf2016semi}. Interestingly, ReLU is the proximal mapping of $I_{\{\mathbf{x}\in\mathbb{R}^{N\times d}|\mathbf{x} \geq \mathbf{O}\}}(\mathbf{x})$.
}
\label{fig:dopgnn}
\end{figure*}

In the following analysis, we denote $\mathbf{x} = (x_1,x_2,\ldots,x_N)^T\in\mathbb{R}^{N\times d}$ and
$S = \{\mathbf{x}\in\mathbb{R}^{N\times d}: x_1 = x_2 = \cdots = x_N\}$. Then,
we  rewrite problem \eqref{prob:composite} as the following compact form:
\begin{equation}\label{prob:compact}
\begin{aligned}
&\underset{\mathbf{x} \in \mathbb{R}^{N\times d}}{\text{minimize}} & & \mathbf{s}(\mathbf{x}) + \lambda\mathbf{r}(\mathbf{x})\\
&\text{subject to}& & \mathbf{x}\in S,
\end{aligned} 
\end{equation}
where $\mathbf{s}(\mathbf{x}) = \sum_{i=1}^N s_i(x_i)$ and  $\mathbf{r}(\mathbf{x}) = \sum_{i=1}^N r_i(x_i)$.

\section{Unrolling Decentralized Optimization Algorithms into Graph Neural Networks }\label{sec:relation and unrolling}

In this section, we introduce classic decentralized optimization algorithms to solve \eqref{prob:compact}  in Section \ref{subsec:decentralized}. We then illustrate the message passing schemes of GNNs and their relationships to decentralized optimization algorithms in Section \ref{subsec:gnn}, and describe the general guidelines of unrolling approaches in Section \ref{subsec:unrolling}.


\subsection{Decentralized Optimization Algorithms}\label{subsec:decentralized}


A  number of decentralized optimization algorithms have been proposed to solve composite optimization problems in the form of \eqref{prob:compact}. To handle the nonsmooth term in problem  \eqref{prob:compact}, one typical approach is to incorporate a proximal operation into certain decentralized first-order algorithms. Below, we present a few classic and state-of-the-art examples of the existing algorithms, which will be embedded into our proposed GNN-based learning framework later and further serve as baselines in our experiments.

\textbf{Prox-DGD:} Prox-DGD \cite{zeng2018proxdgd} incorporates the proximal operator with the decentralized gradient descent (DGD) algorithm \cite{nedic2009dgd}, which is able to handle nonconvex, nonsmooth optimization problems. It operates as follows: Starting from an arbitrary initial point $\mathbf{x}^{0}\in\mathbb{R}^{N\times d}$, and the update equations are given by
\begin{align*}
\mathbf{x}^{k+\frac{1}{2}} = & W\mathbf{x}^k - \alpha^k \nabla \mathbf{s}(\mathbf{x}^k), && \forall k\ge 0,\\
\mathbf{x}^{k+1} =& \prox_{\alpha^k\lambda \mathbf{r}}(\mathbf{x}^{k+\frac{1}{2}}), &&\forall k\ge 0,
\end{align*}
where $\alpha^k> 0$ is the step-size at the iteration $k$  and $W=[w_{ij}]\in\mathbb{R}^{N\times N}$ is the mixing matrix satisfying  \cite[Assumption 2]{zeng2018proxdgd}, which enables the decentralized implementation, i.e., $w_{ij} = 0$ if $i\ne j$ and $\{i,j\}\notin \mathcal{E}$, and $w_{ij}>0$ otherwise.

\textbf{PG-EXTRA:} PG-EXTRA~\cite{shi2015pgextra} is a decentralized proximal gradient method, which is an extension of the decentralized exact first-order algorithm (EXTRA) \cite{shi2015extra}. 
It starts from an arbitrary initial point $\mathbf{x}^0\in\mathbb{R}^{N\times d}$ and then the entire network of agents update as follows:
\begin{align*}
\mathbf{x}^{\frac{1}{2}} =& W\mathbf{x}^{0} - \alpha \nabla \mathbf{s}(\mathbf{x}^0), & & \\
\mathbf{x}^{1} =& \prox_{\alpha\lambda\mathbf{r}}(\mathbf{x}^{\frac{1}{2}});&&\\
\mathbf{x}^{k+1+\frac{1}{2}}=& W \mathbf{x}^{k+1}+\mathbf{x}^{k+\frac{1}{2}}-\tilde{W} \mathbf{x}^{k}&& \nonumber\\
&-\alpha\left[\nabla \mathbf{s}(\mathbf{x}^{k+1})-\nabla \mathbf{s}(\mathbf{x}^{k})\right], &&\forall k\ge 0,\\
\mathbf{x}^{k+2} =& \prox_{\alpha\lambda\mathbf{r}}(\mathbf{x}^{k+1+\frac{1}{2}}), &&\forall k\ge 0,
\end{align*}
where $\alpha>0$ is the step-size. In addition, $W = [w_{ij}]\in\mathbb{R}^{N\times N}$ and $\tilde{W} = [\tilde{w}_{ij}]\in\mathbb{R}^{N\times N}$ are mixing matrices satisfying \cite[Assumption 1]{shi2015pgextra}, which will be discussed in Section \ref{sec:convergence_analysis} detailedly. 


Note that the updates of the above two algorithms can be summarized by the following three procedures at each iteration: (\romannum{1}) Each agent $i$ mixes its decision variable $x_i^k$ with its neighbors' information by particular mixing matrices, when $(x_1^k,\ldots,x_N^k)^T = \mathbf{x}^k$.  (\romannum{2}) Each agent $i$ takes advantage of the  first-order information of $s_i$ (i.e., $\nabla s_i$) to update. (\romannum{3}) Each agent $i$ performs the proximal 
mapping of $r_i$.

We take Prox-DGD as an example to illustrate the above iterative procedure in Figure \ref{fig:dop}, which will be compared with the graph convolutional networks in the next subsection.


\subsection{Graph Neural Networks}\label{subsec:gnn}
In classic machine learning tasks (e.g., image processing), data can be typically embedded into the Euclidean space and thus convolutional neural networks show their competent performance. Recently, there is an increasing number of non-Euclidean applications such as point clouds, molecule structures, and complex networks. These data can be naturally modeled as graphs, and efforts have been put forward to develop neural networks applicable to graphs \cite{wu2019comprehensive}.

Like other deep architectures, graph neural networks (GNNs) have a layer-wise architecture. In each layer of a GNN, every node updates its representation by aggregating features from its neighbors. Specifically, suppose that the graph of a GNN is the same as the aforementioned  \textit{communication graph} $\mathcal{G}$ in Section \ref{sec:sysmodelandmotivations}. Then, the update rule of the $k$-th layer at node $i\in\mathcal{V}$ in a GNN \cite{xu2018powerful} is
\begin{align} \label{eq:gnn}
    x_i^{k+1} = \psi^{k}\left(x_i^{k}, \phi^{k}\left( \left\{ x_j^{k} : j \in \mathcal{N}_i \right\}\right) \right),
\end{align}
where $x_i^k$ denotes the feature vector of node $i$, $\phi^k$ is the aggregation function of node $i$ that collects information from the neighboring nodes, and $\psi^k$ is the function that combines the aggregated information with node $i$'s own information. GNNs in the form of \eqref{eq:gnn} can be implemented efficiently with Pytorch Geometric \cite{fey2019fast}.

One of the popular GNNs is the graph convolutional network (GCN) \cite{kipf2016semi}. The update of a GCN can be written as the following compact form:
\begin{align}
    \mathbf{x}^{k+1} = \text{ReLU}\left(L \mathbf{x}^{k}\Theta  + \mathbf{b}\right),
\end{align}
where ReLU is a nonlinear activation function,  $L\in\mathbb{R}^{N\times N}$ is the Laplacian matrix of the graph $\mathcal{G}$ (which also can induce the decentralized implementation like mixing matrices) and $\Theta\in\mathbb{R}^{d\times d}, \mathbf{b}\in\mathbb{R}^{N\times d}$ are the trainable weights.

The similarities between the GCN and Prox-DGD are shown in Fig \ref{fig:dopgnn}. Note that the operation of both Prox-DGD and PG-EXTRA can be written in the general form of a GNN update as in \eqref{eq:gnn}: Let $\phi^k(\cdot)$ be a \textit{weight sum} aggregation function and $\psi^k(\cdot)$ represents other operations like a gradient descent followed by a proximal mapping.

\subsection{Unrolling Decentralized Algorithms into GNNs}\label{subsec:unrolling}
As classic decentralized algorithms such as Prox-DGD and PG-EXTRA can be cast into the form of GNNs (\ref{eq:gnn}), we may implement these algorithms by the deep graph learning toolbox \cite{fey2019fast}. In the classic optimization algorithms, the hyperparameters need to be manually tuned to achieve desirable performance. With the assistance of data-driven methods, the parameters in the problem formulation and the algorithm (e.g., $\alpha^k$, $\lambda$ and $W$ in the updates of Prox-DGD) can be automatically learned by adopting the recovery error as the loss function. Specifically, denoting $\Theta$ as the set of learnable parameters in the GNN and $\mathbf{x}^{k+1}(\Theta, \{f_i, y_i\}_{i\in\mathcal{V}})$ as the output of the unrolled GNN in the $k$-th layer,  a general form of the loss function can be written as 
\begin{align} \label{loss:gen}
    \ell(\Theta) = d\left(\mathbf{x}^{k+1}(\Theta, \{f_i, y_i\}_{i\in\mathcal{V}}),x^*\right),
\end{align}
where $d(\cdot,\cdot)$ is some distance metric. 

The learning-based algorithm consists of two stages. In the training stage, we collect a number of training samples, i.e., the pair of $(\{f_i, y_i\}_{i\in\mathcal{V}}, x^*)$, and optimize the GNN's learnable parameters by (sub)gradient descent with the objective in \eqref{loss:gen}. In the test stage, the trained neural network is applied and we use the \textit{forward} function in \cite{fey2019fast} to run the algorithms. As the GNNs are implemented by distributed message-passing \cite{fey2019fast}, the update of one node only requires its neighbors' information. As a result, it naturally runs in a decentralized manner.

\section{Application to Decentralized LASSO}\label{sec:algdevelop}
Based on the observations and the general principle developed in the last section, we propose a couple of learning-based methods applying to the decentralized LASSO problem as a running example.
\subsection{LASSO and Algorithms}\label{sec:lasso_alg}
In decentralized LASSO, we have
\begin{align*}
    y_i = A_i x^* + \epsilon_i, \quad \|x^*\|_0 \leq p,
\end{align*}
where $A_i \in \mathbb{R}^{m_i \times d}$ and $m_i$ is the number of measurements for agent $i$. We let $m = \sum_{i=1}^N m_i$ denote the number of total measurements  and further suppose $m_i\ll d$ so that $A_i$ is highly ill-conditioned and each agent $i$ fails to recover $x^*$ independently for $i\in\mathcal{V}$.  This problem has a wide range of applications, e.g., the activity detection in distributed MIMO \cite{guo2020sparse} and distributed feature selection \cite{talwar2015nearly}. To cooperatively estimate the sparse vector $x^*$ in a distributed manner, a common formulation is 


\begin{equation}\label{prob:lasso}
\begin{aligned}
&\underset{x_i \in \mathbb{R}^d\ \forall i \in\mathcal{V}}{\text{minimize}} & & \sum_{i=1}^N \left(\frac{1}{2}\|A_ix_i - y_i\|_2^2 + \lambda \|x_i\|_1\right)\\
&\text{subject to}& & x_1 = \cdots = x_N,
\end{aligned}
\end{equation}
where $x_i$ is the local estimate or inference for agent $i\in\mathcal{V}$.  The equality constraint of \eqref{prob:lasso} enforces the optimal estimates of the agents to achieve a consensus denoted by $\hat{x}\in\mathbb{R}^d$.
 
To match the formulation of \eqref{prob:compact}, we let $\mathbf{s}(\mathbf{x}) = \sum_{i=1}^N \frac{1}{2}{\|A_ix_i-y_i\|_2^2}$ and $\mathbf{r}(\mathbf{x}) = \sum_{i=1}^N \|x_i\|_1$. Then, we are able to apply the following decentralized optimization algorithms mentioned in Section \ref{subsec:decentralized} to the decentralized LASSO, i.e.,  problem \eqref{prob:lasso}.

\textbf{Prox-DGD for LASSO:}  Each agent $i\in\mathcal{V}$ starts from an arbitrary initial point $x_i^0\in\mathbb{R}^d$, and then for each $ k\ge 0$, updates according to
\begin{align} 
    x_i^{k + \frac{1}{2}} = &\sum_{j = 1}^N w_{ij} x^{k}_j  - \alpha^k A_i^T (A_i x_i^{k} - y_i), \label{eq:dgd_lasso0}\\
     x^{k+1}_i = & \mathcal{S}_{\alpha^k\lambda}(x_i^{k + \frac{1}{2}}), \label{eq:dgd_lasso1}
\end{align}
where $\mathcal{S}_a(x)$ is  the soft-thresholding operator with the threshold $aI$. Specifically, 
$\mathcal{S}_{\alpha^k\lambda}(x) = \text{ReLU}(x - \alpha^k \lambda I) - \text{ReLU}( - x - \alpha^k \lambda I)$, which coincides with the ReLU component of the GCN introduced in Section \ref{subsec:gnn}.

\textbf{PG-EXTRA for LASSO:} Each agent $i\in\mathcal{V}$ chooses an arbitrary initial point $x_i^0\in\mathbb{R}^d$  and then for each $k\ge 0$, updates according to \begin{align}x_i^{\frac{1}{2}} =& \sum_{j=1}^N w_{ij}x_j^0\! -\! \alpha A_i^T(A_ix_i^0-y_i), &&\text{if } k=0,\label{eq:pg_lasso_init0}\\
x_i^{1} =& \mathcal{S}_{\alpha\lambda}(x_i^{ \frac{1}{2}}),& & \text{if } k=0;\label{eq:pg_lasso_init1}\\
    x_i^{k + \frac{1}{2}} = &\sum_{j = 1}^N w_{ij} x^{k}_j + x^{k - \frac{1}{2}}_i \!-\! \sum_{j=1}^N \Tilde{w}_{ij} x^{k-1}_j  \notag\\
    & - \alpha  A_i^T A_i (x_i^{k} - x_i^{k-1}) ,  &&\text{if } k>0, \label{eq:pg_lasso0}\\
    x^{k+1}_i = & \mathcal{S}_{\alpha\lambda}(x_i^{k + \frac{1}{2}}),&&\text{if } k>0. \label{eq:pg_lasso1}
\end{align}

Note that Prox-DGD and PG-EXTRA are able to converge to the optimal consensus  $\hat{x}$ of problem \eqref{prob:lasso}, under proper assumptions on the step-sizes and the  mixing matrices. However, the LASSO estimator $\hat{x}$ is different from the ground-truth vector $x^*$ that we expect to recover. In fact, both of the above algorithms perform poorly on LASSO in terms of convergence speed and recovery accuracy in practice, which motivates us to employ learning-based methods to overcome such drawbacks.



\subsection{Recovery Error Analysis}\label{sec:convergence_analysis}
A natural question is what should be learned to improve the performance. To answer this question, we present the recovery error analysis for PG-EXTRA as an example to justify the learning of the step-sizes and the regularization parameters. We impose the standard assumptions in the analysis of decentralized algorithms  \cite{shi2015pgextra} and compressive sensing \cite{chandrasekaran2012convex,oymak2013squared}.

\begin{assumption}\label{assump:mixingmatrix}
The mixing matrices $W = [w_{ij}]\in\mathbb{R}^{N\times N}$ and $\tilde{W} = [\tilde{w}_{ij}]\in\mathbb{R}^{N\times N}$ 
are symmetric and satisfy $[w_{ij}],[\tilde{w}_{ij}]>0\  \forall j\in\mathcal{N}_i\cup \{i\}$, and $[w_{ij}]=[\tilde{w}_{ij}] = 0$ otherwise. Also, $\text{null}\{W-\tilde{W}\} = \text{span}\{\mathbf{1}\}$ and $\text{null}\{I-\tilde{W}\} \supseteq \text{span}\{\mathbf{1}\}$. Moreover, $\tilde{W}\succeq \mathbf{O}$ and $\frac{I+W}{2}\succeq \tilde{W}\succeq W$.
\end{assumption}

\begin{assumption}\label{assump:measurematrix} For $i\in\mathcal{V}$, the entries of
$A_i \sim \mathcal{N}(\mathbf{0}, I)$ and the entries of $\epsilon_i \sim \mathcal{N}(\mathbf{0},\sigma_i^2 I)$.
\end{assumption}

Note that Assumption \ref{assump:mixingmatrix} ensures that each agent only interacts with its neighbors and implies that $\lambda_{\max}(\tilde{W}) = 1$. For the following analysis, we first claim some notations as follows: We denote $\mathbf{x}^* = (x^*,\ldots,x^*)^T\in\mathbb{R}^{N\times d} $ and $\hat{\mathbf{x}} = (\hat{x},\ldots,\hat{x})^T\in\mathbb{R}^{N\times d}$. Suppose that there exists $\hat{\mathbf{q}}\in\mathbb{R}^{N\times d}$ satisfying

 $$(\tilde{W}-W)^{\frac{1}{2}}\hat{\mathbf{q}}+\alpha(\nabla\mathbf{s}(\hat{\mathbf{x}}) + \lambda\tilde{\nabla} \mathbf{r}(\hat{\mathbf{x}})) = \mathbf{O}.$$
Then, let $\mathbf{q}^k = \sum_{t=0}^k (\tilde{W}-W)^{\frac{1}{2}}\mathbf{x}^t$, $\hat{\mathbf{z}} = ((\hat{\mathbf{q}})^T, (\hat{\mathbf{x}})^T)^T$, and $\mathbf{z}^k = ((\mathbf{q}^k)^T, (\mathbf{x}^k)^T)^T$.

The following theorem shows that the recovery error at the $k$-th iteration has an upper bound. 


\begin{theorem} \label{thm:rec}
Suppose that Assumptions \ref{assump:mixingmatrix} and \ref{assump:measurematrix} hold. Let  $L_s = \max_i \lambda_{\max}(A_i^TA_i)$ and $\sigma = \max_{i} \sigma_i$.
If $m > O(p\log(d/p))$ and $\alpha < \frac{2\lambda_{\min}(\tilde{W})}{L_s}$, then
\begin{align*} \label{eq:conv}
    \|\mathbf{x}^{k+1} - \mathbf{x}^*\|_{\tilde{W}}^2  \leq & 2 \|\mathbf{z}^{0} - \hat{\mathbf{z}}\|^2_{G} - 2\sum_{t = 0}^{k}\xi \|\mathbf{z}^{t+1} - \mathbf{z}^{t}\|^2_{G} \nonumber\\
    &+ \mathcal{O}_{\mathbb{P} }\left(\frac{\sigma^2 Np\max\{ \log d,  \lambda^2\}  }{m}     \right),
\end{align*}
where $\xi = 1 - \alpha L_s/(2\lambda_{\min}(\tilde{W}))$ and $G = \left(\begin{array}{cc}
I & \mathbf{O} \\
\mathbf{O} & \tilde{W}
\end{array}\right)$.
\end{theorem}
\begin{proof}
See Appendix \ref{proof:rec}.
\end{proof}

Theorem \ref{thm:rec} provides an error upper bound of the recovery error generated by PG-EXTRA at the $k$-th iteration. To improve the performance, it is better to let the optimal $\alpha$ and $\lambda$ vary across iterations and problem settings. Such flexibility leads to a linear increase of hyperparameters, which are very difficult to fine-tune\footnote{We verify the difficulty of tuning parameters in Appendix \ref{sec:difficult_tune}.}. Therefore, adaptive sequences of step-sizes and regularization coefficients can achieve a better error bound compared to the fixed one.

\subsection{Learning-based Algorithms}\label{sec:learning_based_framework}
Following the general principle in Section \ref{subsec:unrolling}, the remaining part is to determine the learnable parameters and the loss function. Motivated by the recovery error analysis in the last subsection, we set $\alpha^k$ and $\lambda^k$ as the learnable parameters, where $k$ is the iteration index. 

Next, we develop the loss function, which aims to ensure that each step has a sufficient descent. To do so, we set the loss function as a weighted sum of the distances between the $k$-th unrolled iteration's output and the ground-truth $\mathbf{x}^*$, for all $k$, i.e.,
 \begin{equation}\label{eq:loss_func}
   \ell(\Theta) = \sum_{k = 1}^K  \gamma^{K-k} \|\mathbf{x}^k(\Theta) - \mathbf{x}^* \|_F^2,
\end{equation}
where $\Theta = \{\lambda_i, \alpha_i\}_{1 \leq i \leq K}$ groups the learnable parameters, $K$ is the total number of unrolled iterations and $\gamma\le 1$ is a discounting factor.


The learning-based PG-EXTRA is shown in Algorithm \ref{alg}. 

\begin{algorithm}[H]
\caption{Learned PG-EXTRA (LPG-EXTRA)}
 \label{alg}
\begin{algorithmic}[1]
 \STATE Each node $i\in\mathcal{V}$ arbitrarily chooses the initial $x_i^{0}\in\mathbb{R}^d$.
\FOR{$k=0,1,\ldots$}
\STATE $\lambda \gets \lambda_{nn}(k)$.
\STATE $\alpha \gets \alpha_{nn}(k)$.
\STATE Each node $i \in \mathcal{V}$ sends $x_i^k$ to every neighbor $j \in \mathcal{N}_i$.
\STATE Upon receiving $x_j^k$ $\forall j\in \mathcal{N}_i$, each node $i\in \mathcal{V}$ updates  $x_i^{k+\frac{1}{2}}$ and $x_i^{k+1}$  according to \eqref{eq:pg_lasso_init0}-\eqref{eq:pg_lasso1}.
\ENDFOR
\end{algorithmic}
\end{algorithm}
\begin{rem}
The functions $\lambda_{nn}(k)$ and $\alpha_{nn}(k)$ represent the learned parameters outputted by the $k$-th layer of the graph neural network.
\end{rem}
\begin{rem}
The learned Prox-DGD (LProx-DGD) differs from Algorithm \ref{alg} only at the $6$-th line, where the updates of LProx-DGD are given by  \eqref{eq:dgd_lasso0} and  \eqref{eq:dgd_lasso1}.
\end{rem}
\begin{figure*}[htbp]
    \centering
    \subfigure[SNR $=30$ and $m=80$.]
    {
        \includegraphics[width=0.47\columnwidth]{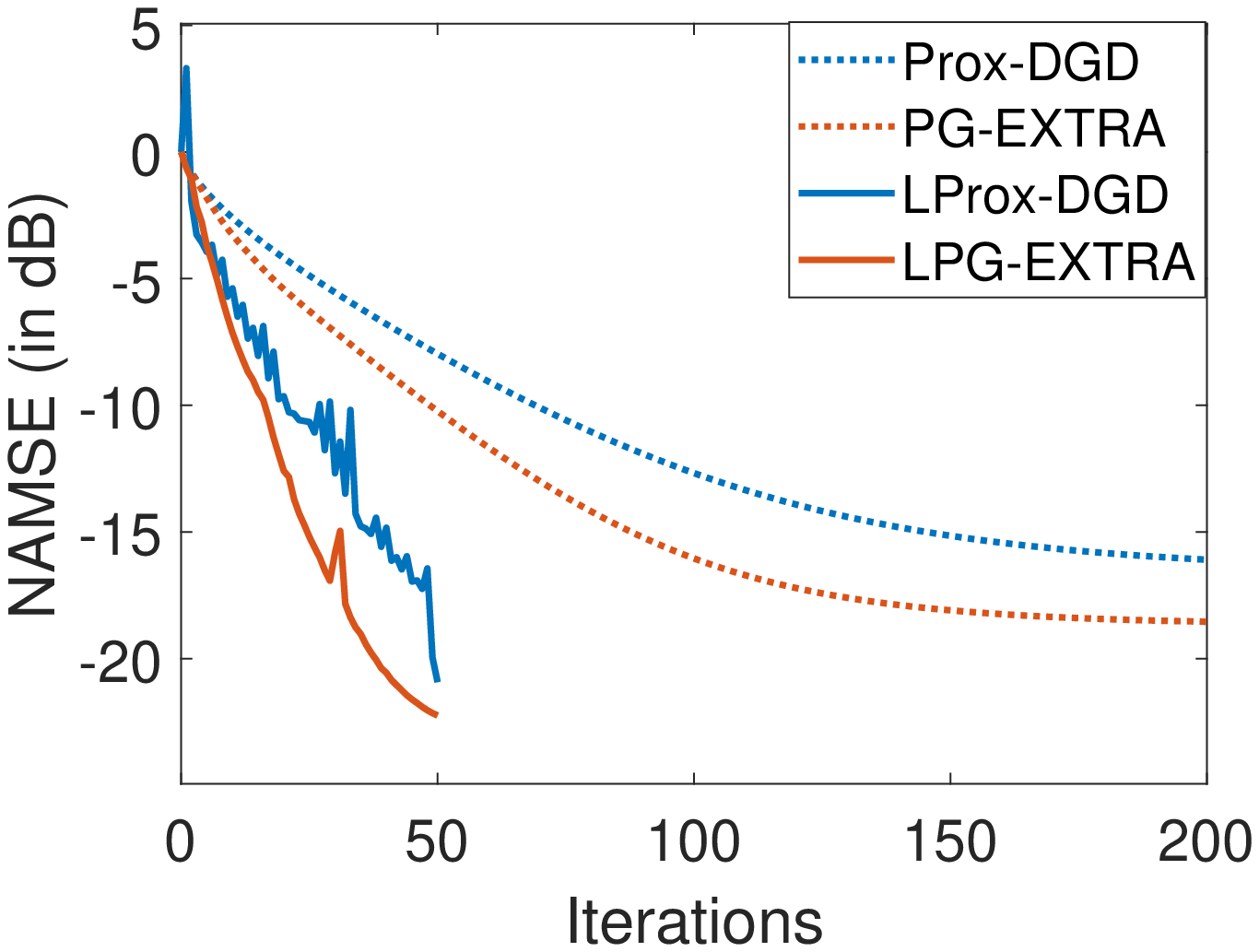}
        \label{fig:m80_con_snr30}
    }
    \subfigure[SNR $= 30$ and $m=300$.]
    {
        \includegraphics[width=0.47\columnwidth]{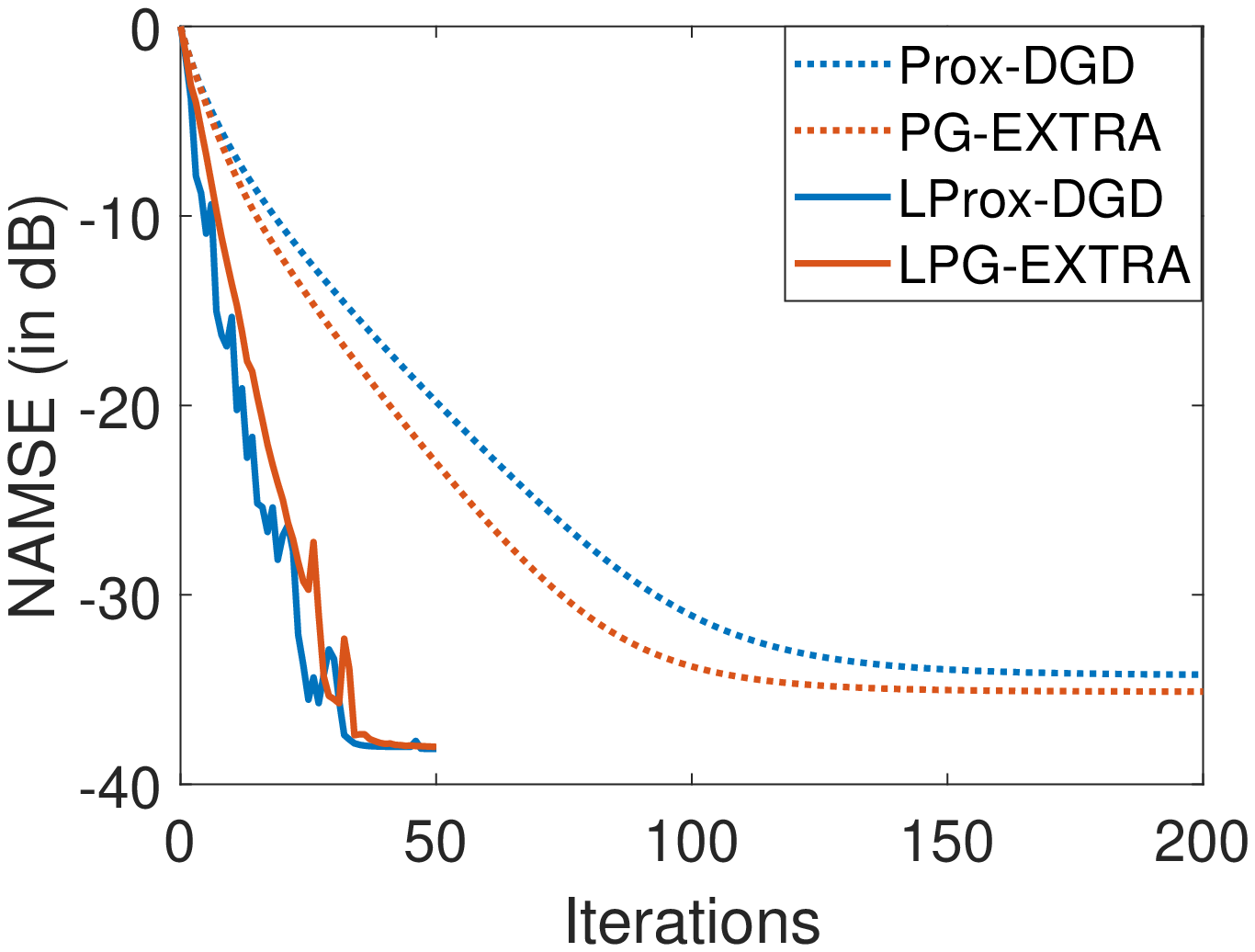}
        \label{fig:m300_con_snr30}
    }
    \subfigure[SNR $=50$ and $m=80$.]
    {
        \includegraphics[width=0.47\columnwidth]{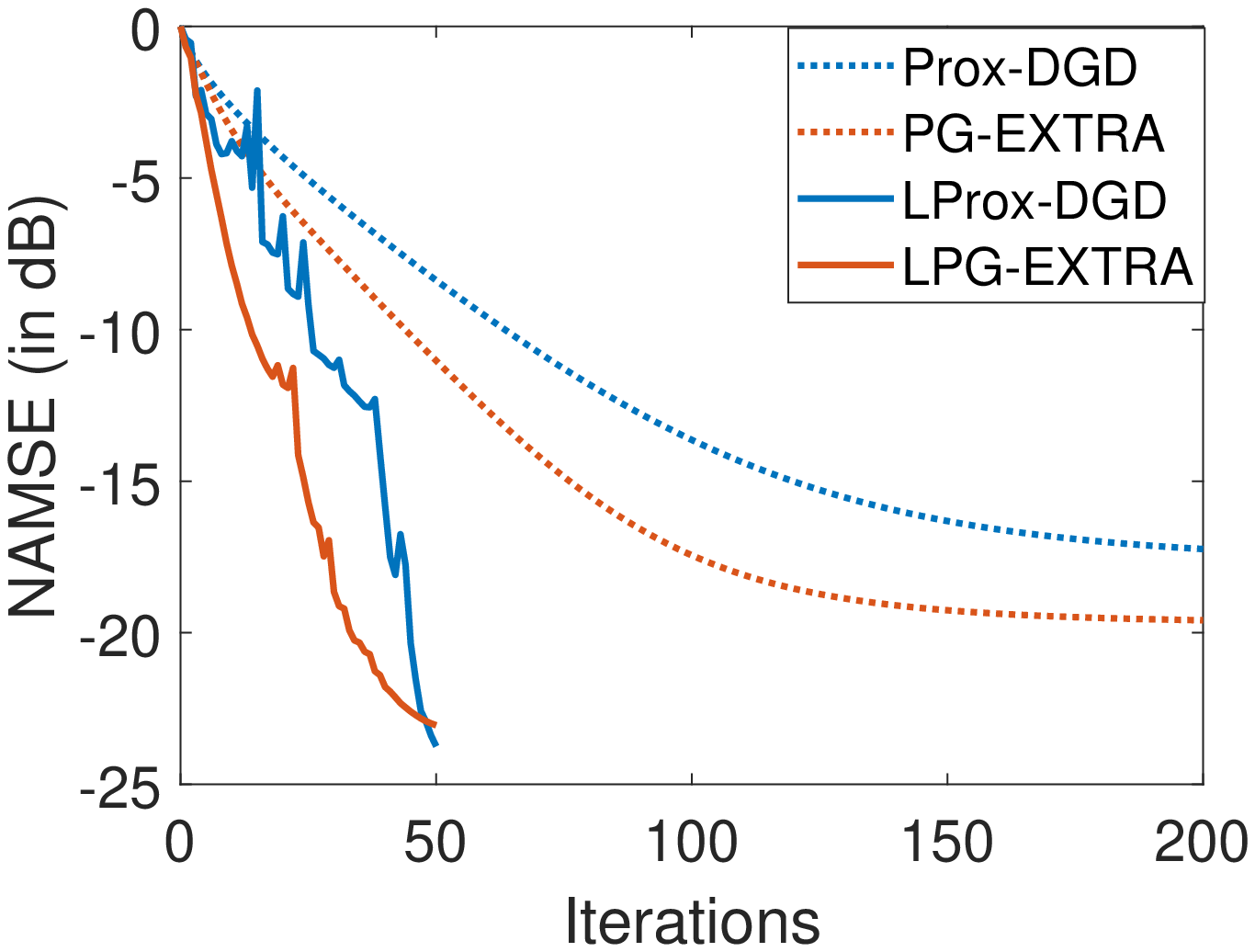}
        \label{fig:m80_con_snr50}
    }
    \subfigure[SNR $=50$ and $m=300$.]
    {
        \includegraphics[width=0.47\columnwidth]{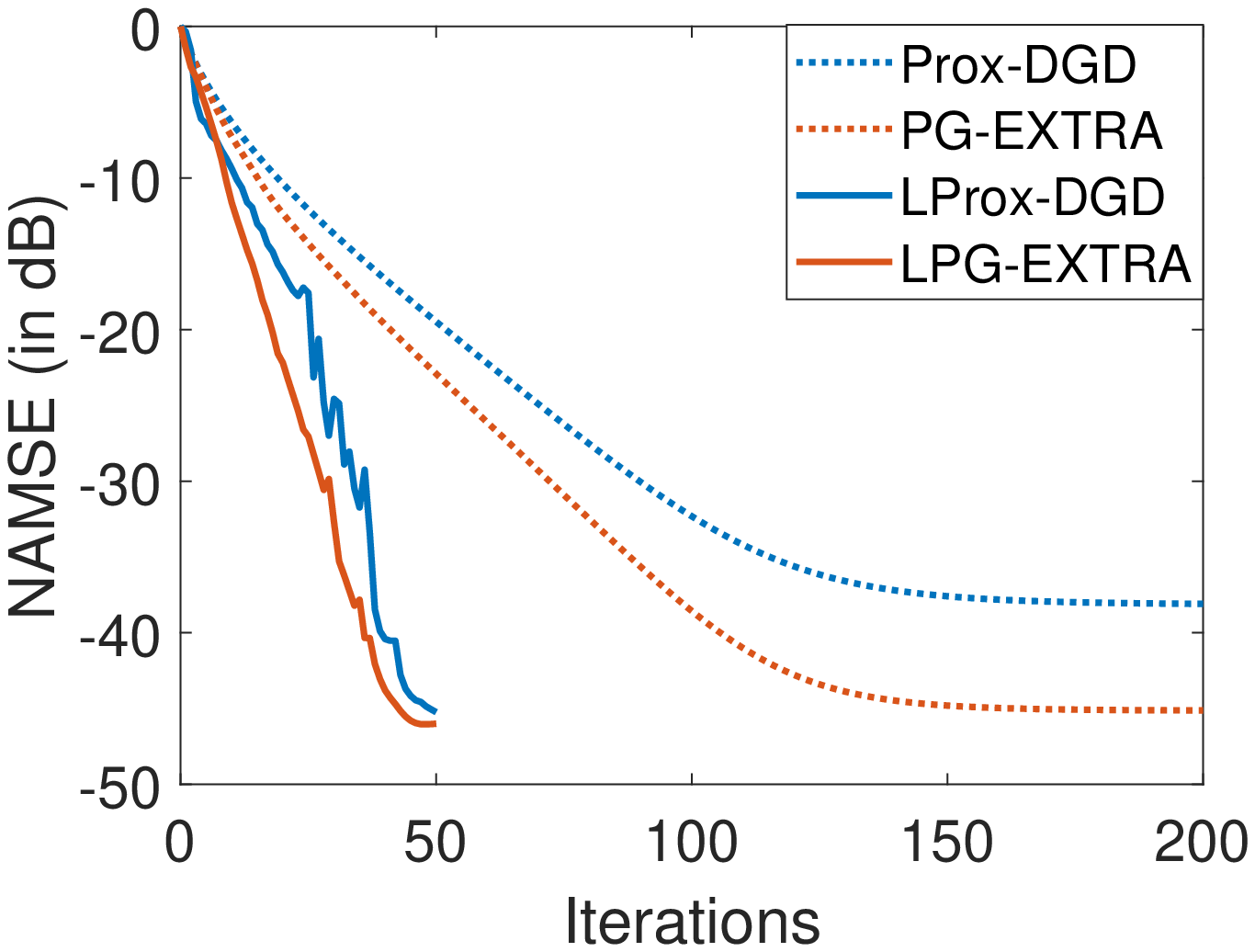}
        \label{fig:m300_con_snr50}
    }
        \caption{Convergence performance of LProx-DGD and LPG-EXTRA compared with the original forms of Prox-DGD and PG-EXTRA with fine-tuned parameters, for SNR $= 30,50$ and $m=80,300$, respectively. }
    \label{fig:convergence}
\end{figure*}

During the training, the ground-truth vector $x^*$ is required to be available for all the agents to learn the parameters, but it is definitely not needed for testing. Guided by minimizing the loss function, agents are enforced to emulate the ground-truth vector within $K$ iterations. In the sequel, we illustrate via a number of experiments that the learned algorithms perform excellently for a family of input measurements and ground-truth vectors.

\begin{figure*}[htbp]
    \centering
    \subfigure[Prox-DGD and LProx-DGD ($m = 80$).]
    {
        \includegraphics[width=0.47\columnwidth]{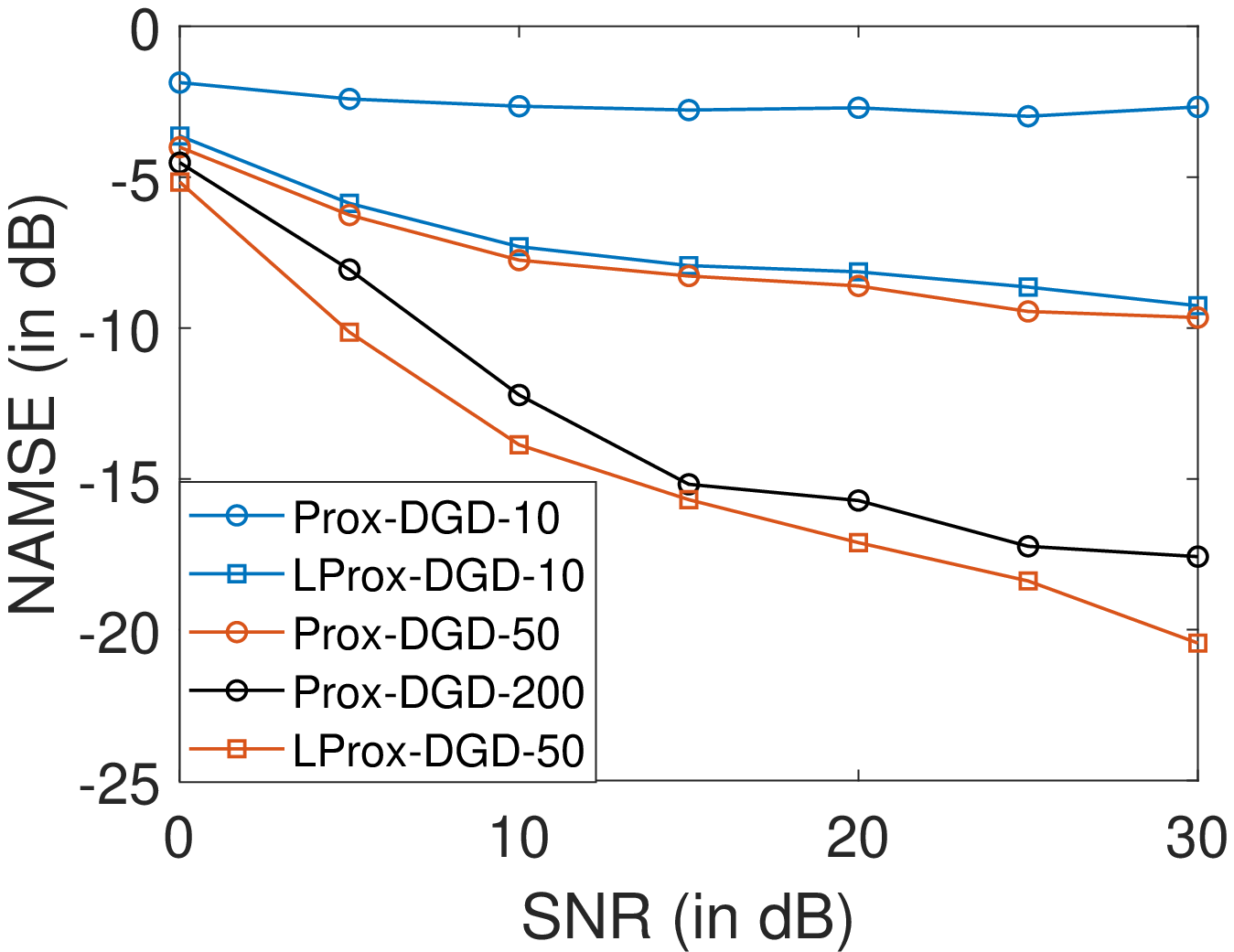}
        \label{fig:m80_snr_dgd}
    }
    \subfigure[Prox-DGD and LProx-DGD ($m = 300$).]
    {
        \includegraphics[width=0.47\columnwidth]{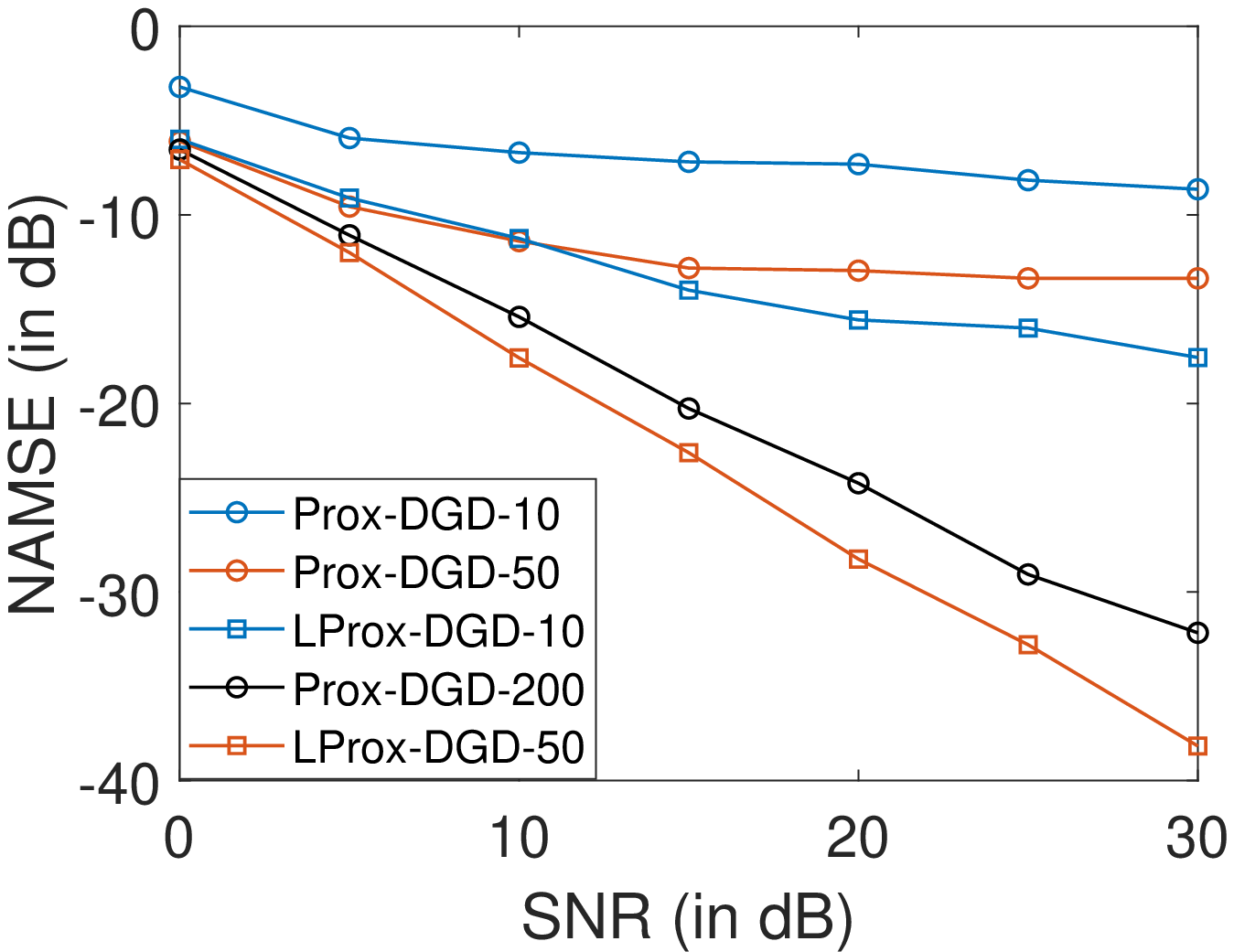}
        \label{fig:m300_snr_dgd}
    }
    \subfigure[PG-EXTRA  and LPG-EXTRA ($m = 80$).]
    {
        \includegraphics[width=0.47\columnwidth]{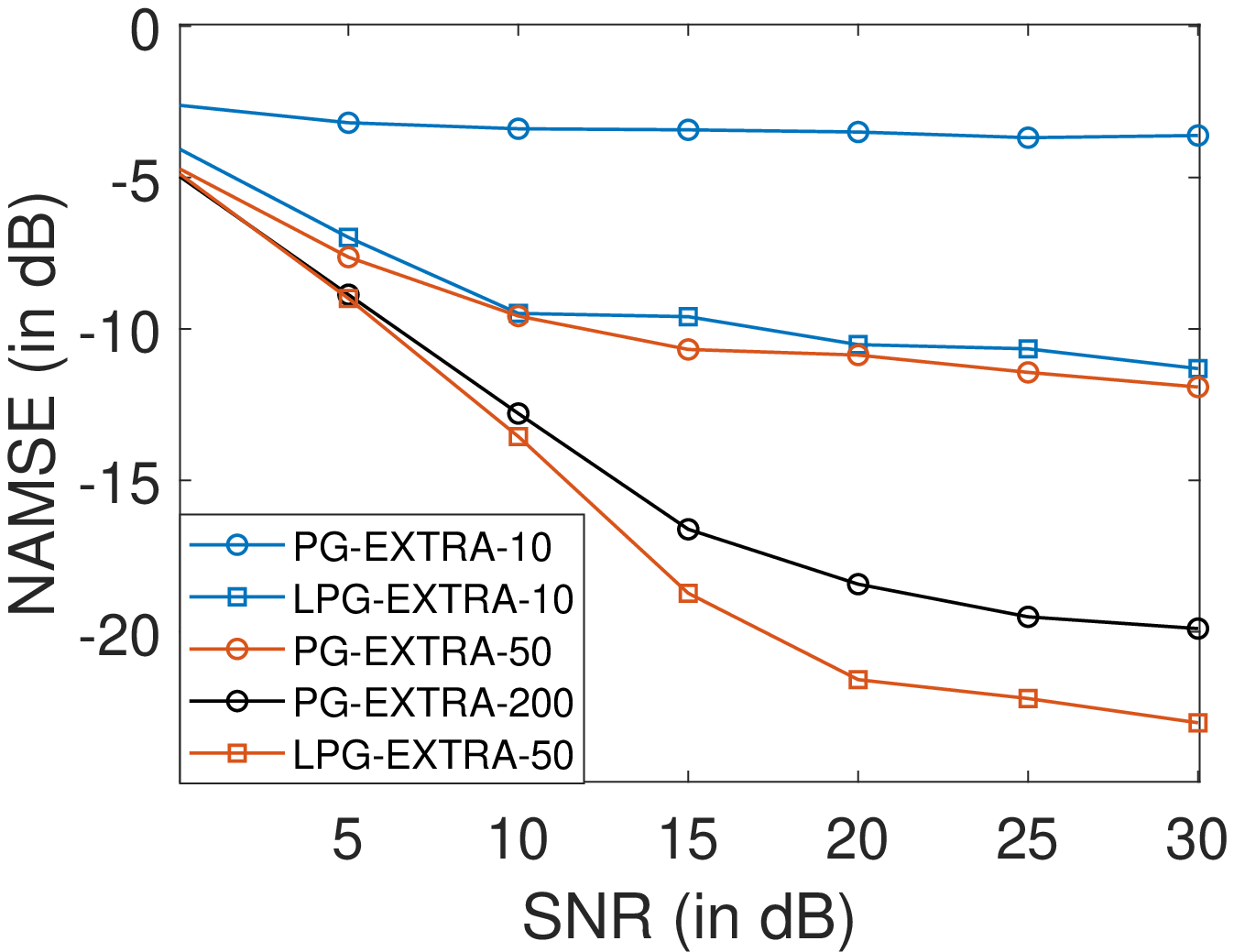}
        \label{fig:m80_snr_pg}
    }
    \subfigure[PG-EXTRA and LPG-EXTRA ($m = 300$).]
    {
        \includegraphics[width=0.47\columnwidth]{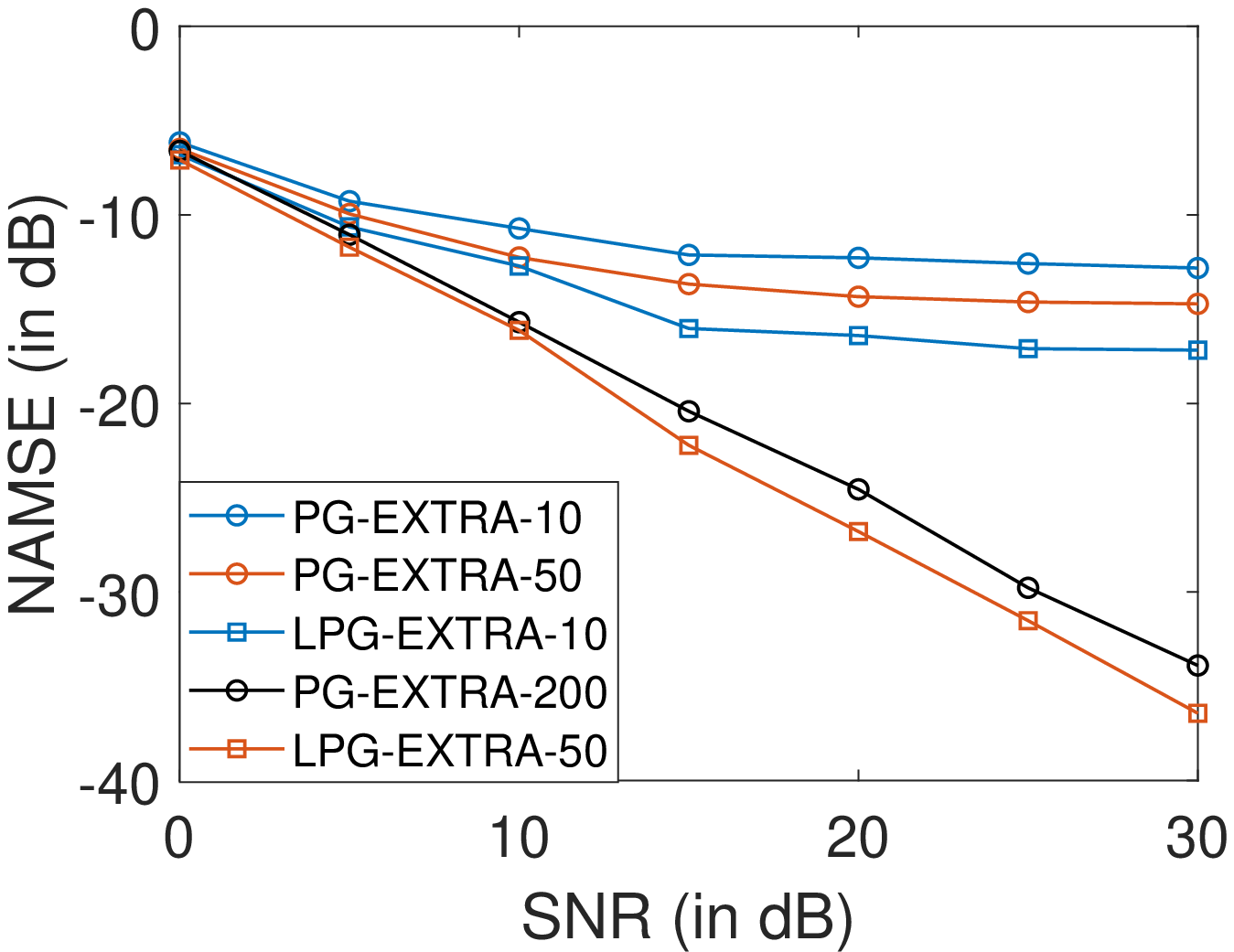}
        \label{fig:m300_snr_pg}
    }
        \caption{Performance of  LProx-DGD and LPG-EXTRA, compared with Prox-DGD, PG-EXTRA with fine-tuned parameters, under different SNRs. The left two figures show the final recovery errors (evaluated by NAMSE) of LProx-DGD and Prox-DGD, while the right two figures present those of  LPG-EXTRA and PG-EXTRA, with different unrolled layers or iterations.}
    \label{fig:SNR}
\end{figure*}
\section{Experiments}\label{sec:experiments}
In this section, we conduct experiments to demonstrate the effectiveness of LProx-DGD and LPG-EXTRA, compared with  Prox-DGD and PG-EXTRA as the baselines, in solving the decentralized sparse vector recovery problem \eqref{prob:lasso}. 

\subsection{Settings}
For the following experiments, the dataset contains 1000 samples for training, 100 samples for validation, and 100 samples for testing. In each sample, the communication graph is generated  randomly by choosing from the set of all graphs with $\lvert \mathcal{V}\rvert = N=5$ nodes and $\lvert \mathcal{E}\rvert = 6$ edges. We let $m\in\{80,300\}$, $m_i = m/N$ and $d_i=100$ and then sample the entries of the sensing matrix $A_i\in \mathbb{R}^{m_i\times 100}\sim\mathcal{N}(\mathbf{0},I)$ and the entries of noise $\epsilon_i\sim \mathcal{N}(\mathbf{0},\sigma^2 I)$ for each agent $i\in\mathcal{V}$, where $\sigma$ is related to the signal-to-noise ratio (SNR) defined as
\begin{align*}
    \text{SNR} = 10 \log_{10}\left(\frac{\mathbb{E}\|x^*\|_2^2 }{\sigma^2} \right).
\end{align*} The sparse ground-truth vector $x^*$ is decided to be nonzero entries following a  Bernoulli-Gaussian distribution with $p_s = m/(2N)$, while the values of nonzero entries also follow the standard normal distribution. 

The recovery error is evaluated by the Normalized Average Mean Square Error (NAMSE) in dB:

\begin{equation}
   \text{NAMSE}(\mathbf{x}^k,\mathbf{x}^*) =10\log_{10}\frac{\mathbb{E}\|\mathbf{x}^k-\mathbf{x}^*\|_F^2} {N \cdot\mathbb{E}\|\mathbf{x}^*\|_F^2},
\end{equation}
for the $k$-th unrolled layer or iteration. We employ NAMSE as the evaluation metric of the following experiments.

We train the learning-based methods for 500 epochs with batch training samples, using the Adam\cite{kingma2015adam} method. For the optimization-based methods, we let the metropolis matrix be $W$ for Prox-DGD and further set $\tilde{W} = \frac{I+W}{2}$ for PG-EXTRA in the following experiments. 

\subsection{Convergence Performance}

We use 50-unrolling-step LProx-DGD and LPG-EXTRA, and compare them to the original forms of Prox-DGD and PG-EXTRA. Since the learning-based methods are data-driven, we fine-tune the fixed parameters in Prox-DGD and PG-EXTRA using validation samples for all the following experiments. However, the process of fine-tuning by grid search methods is tedious and expensive, since the performance is sensitive to the step-size $\alpha$ and the regularization parameter $\lambda$ (cf. Appendix \ref{sec:difficult_tune}).

Figure \ref{fig:convergence}  plots the convergence performance of  Prox-DGD, PG-EXTRA, LProx-DGD and LPG-EXTRA in terms of NAMSE, for different SNRs and the number of measurements. Observe that LProx-DGD and LPG-EXTRA perform similarly, while both of them significantly outperform  Prox-DGD and PG-EXTRA in  convergence speed. When the number of total measurements is sufficient and the noise level is low (cf. Figure \ref{fig:m300_con_snr50}), the learning-based methods  achieve similar level recovery errors with much fewer iterations. When the number of total measurements is insufficient or the noise level is medium (cf. Figures \ref{fig:m80_con_snr30}-\ref{fig:m80_con_snr50}), the learning-based methods can achieve both faster convergence and smaller recovery errors. This is owing to the learning of a tailored and iteration-wise $\lambda$.


\subsection{Robustness to Noise Levels}
To show that the promising performance of LProx-DGD and LPG-EXTRA is robust to the various levels of noises, we compare their final NAMSEs at different SNRs for recovery problems with different sizes. For each of the following experiments, we unroll each of the learning-based algorithms with $10$ layers and $50$ layers respectively, and compare them to their original forms for $10$, $50$ and $200$ iterations. The results are shown in Figure \ref{fig:SNR}.

Figures \ref{fig:m80_snr_dgd} and \ref{fig:m300_snr_dgd} plot the final recovery errors of LProx-DGD and Prox-DGD with different unrolled iterations at various SNRs, for $m=80$ and $m=300$ respectively.  Similarly, Figure \ref{fig:m80_snr_pg} and \ref{fig:m300_snr_pg} plot those of LPG-EXTRA and PG-EXTRA. Observe that the final recovery errors of LProx-DGD and LPG-EXTRA with $10$ (resp. $50$) unrolled layers are comparable to those of Prox-DGD and PG-EXTRA with $50$ (resp. $200$) iterations, at different levels of noises in both cases of $m=80$ and $m=300$.

Such appealing performance substantiates that the learned parameters can remarkably reduce the number communication rounds by up to $80\%$ compared to the original forms, while guaranteeing the same level of recovery errors given different noisy measurements.

\subsection{Impact of Training Samples}
Deep learning-based methods often require a large number of training samples, which is infeasible to obtain in practice. In this subsection, we show the impact of sample size  on the performance, under the setting of $m=300$ in the last two subsections. As is revealed in Table \ref{tab:sample}, the trained model performs well with only tens of samples. This is achieved by introducing a strong inductive bias into the model.

\begin{table}[htbp]
	
	\selectfont  
	\centering
	
	\caption{Performance (NAMSE in dB) for  LProx-DGD and LPG-EXTRA with different numbers of training samples.} 
	
	\resizebox{0.46\textwidth}{!}{
		\begin{tabular}{|c|c|c|c|c|c|c|c|}  
			\hline  
 Training Samples & 1 & 5 & 10 & 20 & 50 & 100 & 1000 \\ \hline
 LProx-DGD-10 & -9.99 & -14.17 & -14.74 & -14.59 & -14.77 & -16.98 & -17.65 \\ \hline
 LPG-EXTRA-10 & -11.74 & -14.04 & -14.87 &    -15.19 & 15.41 & -16.73 & -17.17\\ \hline
 LProx-DGD-50 & -30.98 & -34.51 & -34.75 & -35.39 & -35.87 & -36.22 & -38.17 \\ \hline
 LPG-EXTRA-50 & -32.57 & -34.77 & -35.07 &   -35.71 &  -35.96 & -36.07& -36.43\\ \hline

	\end{tabular}}
	\label{tab:sample}
\end{table}

\section{Conclusion}\label{sec:conclusion}
In this paper, we have developed a learning-based framework for decentralized statistical inference problems. A main ingredient is the adoption of classic first-order decentralized optimization algorithms as the inductive bias for GNNs, which results in both theoretical guarantees and competitive practical performance. Moreover, this work has identified the connection between the GNN architectures and the decentralized proximal methods, which is of independent interest. This investigation may trigger profound implications in both theoretical and practical aspects of decentralized optimization and learning.

\bibliographystyle{IEEEtran}
\bibliography{ref_unabbr}
\appendix
\subsection{Proof for Theorem \ref{thm:rec}}\label{proof:rec}

This proof is based on the convergence analysis of PG-EXTRA \cite{shi2015pgextra} and the MSE analysis of LASSO \cite{oymak2013squared}. Specially, by denoting
 $T := \mathcal{O}_{\mathbb{P} }\left(\frac{\sigma^2 Np\max\{ \log d,  \lambda^2\}  }{m}     \right)$, we have
    \begin{align*}
        &\|\mathbf{x}^{k+1} - \mathbf{x}^*\|^2_{\tilde{W}}  
        \leq 2\|\mathbf{x}^{k+1} - \hat{\mathbf{x}}\|^2_{\tilde{W}} + 2\| \hat{\mathbf{x}} - \mathbf{x}^*\|^2_{\tilde{W}} \\
        \leq & 2\|\mathbf{z}^{k+1} - \hat{\mathbf{z}}\|^2_{G} + 
        2\|\tilde{W}\|_2^2\| \hat{\mathbf{x}} - \mathbf{x}^*\|^2_{F} \\
        \overset{(a)}{\leq}& 2\|\mathbf{z}^{k+1} - \hat{\mathbf{z}}\|^2_{G} + 
         2T \\
        \overset{(b)}{\leq} & 2\left(\|\mathbf{z}^{k} - \hat{\mathbf{z}}\|^2_{G} - \xi \|\mathbf{z}^{k+1} - \mathbf{z}^{k}\|^2_{G}\right) + 2T \\
        \leq & 2 \left( \|\mathbf{z}^{0} - \hat{\mathbf{z}}\|^2_{G} - \sum_{t = 0}^{k}\xi \|\mathbf{z}^{t+1} - \mathbf{z}^{t}\|^2_{G} \right) + 2T,\\
    \end{align*}
    where the inequality (a) is derived from Lemma \ref{lem:pg_extra} and $\lambda_{\max}(\tilde{W}) = 1$ and the inequality (b) is due to Lemma \ref{lem:rec_error}.
    
\begin{lemma}[\cite{shi2015pgextra}]\label{lem:pg_extra}
    Suppose Assumption \ref{assump:mixingmatrix} holds. If $\alpha\in\left(0,\frac{2\lambda_{\min}(\tilde{W})}{L_s}\right)$, then the sequence $\mathbf{z}^k$ generated by PG-EXTRA satisfies 
    \begin{align} \label{eq:seq_extra}
        \|\mathbf{z}^k - \hat{\mathbf{z}}\|^2_G - \|\mathbf{z}^{k+1} - \hat{\mathbf{z}}\|^2_G \geq \xi \|\mathbf{z}^k - \mathbf{z}^{k+1}\|^2_G,
    \end{align}
    where $\xi = 1 - \frac{\alpha L_s}{2\lambda_{\min}(\tilde{W}) }$ and $k\ge 0$.
\end{lemma}

\begin{lemma} \label{lem:rec_error}
    Suppose Assumption \ref{assump:measurematrix} holds and $m > O(p \log(d/p))$. Then, there exists $C>0$ such that with probability $1-\exp(-Cm)$, we have
    \begin{align*}
        \|\hat{\mathbf{x}} - \mathbf{x}^*\|^2_F \leq \mathcal{O}_{\mathbb{P} }\left(\frac{\sigma^2 Np\max\{ \log d,  \lambda^2\}  }{m}     \right).
    \end{align*}

\end{lemma}

\begin{proof}
    A straightforward integration of Proposition 3.10 in \cite{chandrasekaran2012convex}, Theorem 3.2 and Lemma 11.1 in \cite{oymak2013squared}.
\end{proof}

\subsection{The Difficulty in Tuning Parameters}\label{sec:difficult_tune}

Conventionally, the parameters in PG-EXTRA for solving problem \eqref{prob:composite} including the regularization parameter $\lambda$ and the step-size $\alpha$ can be tuned manually, provided that the ground-truth vectors in validation samples are available. We take the sparse vector recovery problem \eqref{prob:lasso} as an example to demonstrate the difficulty in tuning parameters.

\begin{table}[htbp]
	
	\selectfont  
	\centering
	
	\caption{Performance of PG-EXTRA after $200$ iterations under different parameters  for solving  problem \eqref{prob:lasso} (NAMSE in dB).} 
	
	\resizebox{0.45\textwidth}{!}{
		\begin{tabular}{|c|c|c|c|c|c|c|}  
			\hline  
 \backslashbox{$\lambda$}{$\alpha$}& 0.001 & 0.003 & 0.004 & 0.005 & 0.006 \\ \hline
 0.05 & -7.61 & -12.71 & -15.29 & -18.14 & 2.15 \\ \hline
 0.1 & -8.28 & -16.59 & -21.54 &   \textbf{-26.44} &  -14.67\\ \hline
 0.3 & -10.60 & -23.07 & -24.50 & -24.78 & -24.84 \\ \hline
 0.5 & -11.98  & -20.37 & -20.67 &  -20.71 & -20.71 \\ \hline
			
	\end{tabular}}
	\label{tab:tune_pgextra}
\end{table}

Table \ref{tab:tune_pgextra} shows that the recovery performance is sensitive to the parameters, in the case where the number of total measurements $m=100$. For example, the best NAMSE  -26.44 is achieved when $\alpha=0.005$ and $\lambda=0.1$. Nevertheless, once the step-size $\alpha$ is increased or decreased by 0.001, the recovery performance degrades a lot. Likewise, 
a sight change of the regularization parameter $\lambda$ would also lead to the large difference in NAMSEs. Therefore, it may cause prohibitively expensive costs to fine-tune the parameters that lead to satisfactory and robust performance of the algorithms.


\end{document}